\documentclass[3p]{elsarticle}

\usepackage{lineno,hyperref}
\usepackage{optidef}
\usepackage{amsmath}
\usepackage{amssymb}
\usepackage{amsthm}
\usepackage{xcolor}

\usepackage{microtype}
\usepackage{graphicx}
\usepackage{subcaption}
\usepackage{booktabs}

\usepackage{fancyhdr}
\usepackage{algorithm}
\usepackage{algorithmic}
\usepackage{natbib}
\usepackage{eso-pic} % used by \AddToShipoutPicture 
\usepackage{forloop}

\DeclareMathOperator{\argmin}{arg\,min}
\DeclareMathOperator{\sgn}{sgn}

\newcommand{\C}{\mathcal{C}}

\newcommand{\AAS}{\mathcal{E}_{T, \delta}}
\setcitestyle{square}

\renewcommand{\vec}[1]{\boldsymbol{\mathbf #1}}

\newtheorem{theorem}{Theorem}[section]

\newtheorem{lemma}{Lemma}
\makeatletter
\newcommand{\printfnsymbol}[1]{%
  \textsuperscript{\@fnsymbol{#1}}%
}
\modulolinenumbers[5]

\journal{Pattern Recognition}
\begin{document}

\begin{frontmatter}

\title{A black-box adversarial attack for poisoning clustering}
\author[rvt]{Antonio Emanuele Cinà\corref{cor1}\fnref{fn1}}%$^\dagger$}
\ead{antonioemanuele.cina@unive.it}

\author[rvt]{Alessandro Torcinovich}%$^\dagger$}
\ead{ale.torcinovich@unive.it}
\author[rvt]{Marcello Pelillo\corref{cor2}}
\ead{pelillo@unive.it}

\cortext[cor1]{Corresponding author}
\cortext[cor2]{Principal corresponding author}

\address[rvt]{Ca’ Foscari University of Venice, Italy}

\begin{abstract}
Clustering algorithms play a fundamental role as tools in decision-making and sensible automation processes. Due to the widespread use of these applications, a robustness analysis of this family of algorithms against adversarial noise has become imperative. To the best of our knowledge, however, only a few works have currently addressed this problem. In an attempt to fill this gap, in this work, we propose a black-box adversarial attack  for crafting adversarial samples to test the robustness of clustering algorithms. We formulate the problem as a constrained minimization program, general in its structure and customizable by the attacker according to her capability constraints. We do not assume any information about the internal structure of the victim clustering algorithm, and we allow the attacker to query it as a service only. In the absence of any derivative information, we perform the optimization with a custom approach inspired by the Abstract Genetic Algorithm (AGA). In the experimental part, we demonstrate the sensibility of different single and ensemble clustering algorithms against our crafted adversarial samples on different scenarios. Furthermore, we perform a comparison of our algorithm with a state-of-the-art approach showing that we are able to reach or even outperform its performance. Finally, to highlight the general nature of the generated noise, we show that our attacks are transferable even against supervised algorithms such as SVMs, random forests and neural networks.
\end{abstract}


\end{frontmatter}

\section{Introduction}
The state of the art in machine learning and computer vision has greatly improved over the course of the last decade, to the point that many algorithms are commonly used as effective aiding tools in security (spam/malware detection \cite{SpamMailDetection}, face recognition \cite{DBLP:conf/iccv/TapaswiLF19}) or decision making (road-sign detection \cite{stopSignSVM}, cancer detection \cite{cancerPrognosis}, financial sentiment analysis \cite{DBLP:journals/jbd/SohangirWPK18}) related tasks. The increasing pervasiveness of these applications in our everyday life poses an issue about the robustness of the employed algorithms against sophisticated forms of non-random noise.

Adversarial learning has emerged over the past few years as a line of research focused on studying and addressing the aforementioned robustness issue. Perhaps, the most important result in this field is the discovery of \emph{adversarial noise}, a wisely crafted form of noise that, if applied to an input, does not affect human judgment but can significantly decrease the performance of the learning models \cite{SzegedyProperties, RobinAdversarial}. Adversarial noise has been applied with success to fool models used in security scenarios such as spam filtering \cite{lowdSpamFilters, battistaSpamFilters, AbdolrahmanspamFilter} or malware detection \cite{isDataClustering}, but also in broader scenarios such as image classification \cite{stopSign}.

The vast majority of the works done so far in this field deals with supervised learning. However, its unsupervised counterpart is equally present in sensible applications, such as fraud detection, image segmentation, and market analysis, not to mention the plethora of security-based applications for detecting dangerous or illicit activities %\cite{androidMalware, networkTraffic, malwareClust}. 
\cite{dnsTraffic, androidMalware, networkTraffic, malwareClust, ec2Malware}. 

It follows that the robustness of unsupervised algorithms used by those applications is crucial to give credibility to the results provided. Among the unsupervised tasks, in this work, we focus our attention on instance clustering with feature and image data. 

The majority of clustering algorithms are not differentiable, thus adversarial gradient-based approaches -- widely used in supervised settings -- are not directly applicable. Since, in general, the machine learning field is currently dominated by gradient-based methods, this may represent a possible reason for the limited interest in this field. Nonetheless, the problem has been addressed in a complete white-box setting in \cite{HidingSkillicorn, isDataClustering, dbscanAttack}, where some gradient-free attack algorithms have been proposed. In these works, the authors usually leverage the internal behavior of the clustering methods under study to craft \textit{ad-hoc} adversarial noise. To the best of our knowledge, little work has been done against black-box algorithms. The design of black-box adversarial attacks, not only can help in finding common weaknesses of clustering algorithms but can also pave the road towards general rules for the formulation of robust clustering algorithms. In this work, we propose an algorithm to craft adversarial examples in a gradient-free fashion, without knowing the identity of the target clustering method. We assume that the attacker can only perform queries to it.
Furthermore, we argue that, due to its general nature, the noise generated by our adversarial algorithm can also be applied to fool effectively supervised methods.

The main contributions of our work are as follows: (a) We design a new black-box gradient-free optimization algorithm to fool data clustering algorithms, and provide convergence guarantees. (b) We propose a new objective function that takes into consideration the attacker's capability constraints, motivating its suitability in this setting. (c) We perform experiments on three different datasets against different clustering methods, showing that our algorithm can significantly affect the clustering performance. (d) Following the work of \cite{transferability}, we perform a transferability analysis, showing that our crafted adversarial samples are suitable to fool supervised algorithms.

\section{Related Work}
Several works use clustering for extracting data patterns in a given dataset. For instance, the work of \cite{Malheur} proposed \textit{Malheur}, a tool for behavioral malware detection that combines clustering and classification for detecting novel malware categories. \cite{androidMalware} have proposed \textit{AnDarwin}, a software for detecting plagiarism in Android applications. In this approach, clustering is used to handle large numbers of applications, unlike previous methods that compare apps pairwise. More recently, \cite{ipclustering} have presented a tool for anomaly detection in networking by using a clustering algorithm. Despite the greater need to have robust clustering algorithms, only a few works address their security problems.

The first works on the analysis of adversarial manipulations against clustering algorithms were proposed in \cite{HidingSkillicorn} and \cite{FirstAdversarialClustering}. The authors observed that some samples could be misclustered by positioning them close to the original cluster boundary, so that a new \emph{fringe} cluster is formed.

\cite{isDataClustering} provided a theoretical formulation for the adversarial clustering problem and proposed a perfect-knowledge attack to fool single-linkage hierarchical clustering. In particular, the authors defined two different attack strategies: poisoning and obfuscation. The former infects data to violate the system availability and deteriorate the clustering results. The latter taints a target set of samples to violate the system integrity. In our threat model, we share the same aims of the poisoning strategy; however, differently from what has been done in \cite{isDataClustering}, we extend the application of poisoning by allowing the attacker to manipulate already existing samples in the dataset instead of injecting new ones.
Later on, in \cite{poisoningComplete}, the previous work was extended by proposing a threat model against complete-linkage hierarchical clustering. \cite{dbscanAttack} defined a threat algorithm to fool DBSCAN-based algorithms by selecting and then merging arbitrary clusters.

All the aforementioned works assume that the attacker has perfect knowledge about the clustering algorithm under attack. In our work, we overcome this assumption by proposing a gradient-free algorithm to fool clustering algorithms in a generalized black-box setting, meaning that the attacker has no prior knowledge about the clustering algorithm and its parameters. We design our algorithm as an instance of an Abstract Genetic Algorithm \cite{DBLP:conf/ppsn/EibenAH90}, in which the adversarial noise improves generation by generation.

Recently, a similar problem has been addressed in \cite{suspicion}, where it has been proposed a derivative-free, black-box attack strategy to target clustering algorithms working on linearly separable tasks. The approach consists of manipulating only one specific input sample feature-by-feature to corrupt the clustering decision boundary. However, our method is still different, since:
\begin{itemize}
    \item It has been designed for attacking generic clustering algorithms (not only linearly separable ones).
    
    \item We propose a way to address multi-clustering problems by allowing the attacker to manipulate samples coming from different clusters.
    
    \item We prove that our algorithm has significant convergence properties to find the optimal perturbation for multiple samples and features at the same time.
    
    \item We penalize our solutions by considering the number of manipulated features in addition to the maximum acceptable noise threshold.
\end{itemize}

\section{Methodology}
Let $\vec{X} \in \mathbb{R}^{n\times d}$ denote a feature matrix representing the dataset to be poisoned, where $n$ is the number of samples and $d$ is the number of features. We define $\C: \mathbb{R}^{n \times d} \rightarrow \{1, \dots, K\}^n$ to be the target clustering algorithm, that separates $n$ samples into $K$ different classes ($1 \le K \le n$). 
%An important observation is that the attacker can retrieve $K$ by merely querying the target clustering algorithm, also in the case in which $K$ varies.
We remark that by querying the clustering algorithm, the attacker can retrieve the number of clusters, and they may also change during the evaluation.

We consider the problem of crafting an adversarial mask $\vec{\epsilon}$, to be injected into $\vec{X}$, such that the clustering partitions $\C(\vec{X})$ and $\C(\vec{X} + \vec{\epsilon})$ are different to a certain degree. In real scenarios, the attacker may follow some policies on the nature of the attack, usually imposed by intrinsic constraints on the problem at hand \cite{DBLP:journals/pr/BiggioR18}. We model the scenario in which the attacker may want to perturb a specific subset of samples $T \subseteq \{1, \dots, n\}$, in such a way that the attack is less human-detectable. i.e. by constraining the norm of $\vec{\epsilon}$ \cite{DBLP:conf/cvpr/Moosavi-Dezfooli17, DBLP:journals/tec/SuVS19}. %\cite{DBLP:conf/cvpr/Moosavi-Dezfooli17, DBLP:conf/cvpr/DongLPS0HL18, DBLP:journals/tec/SuVS19}.
In our work, the attacker's capability constraints \cite{DBLP:journals/pr/BiggioR18} are thus defined by (a) an \emph{attacker's maximum power} $\delta$, which is the maximum amount of noise allowed to be injected in a single entry $\vec{x}_{ij}$, (b) an \emph{attacker's maximum effort} $\gamma$, which is the maximum number of manipulable entries of $\vec{X}$. Further, we assume the attacker has access to the feature matrix $\vec{X}$, and she can query the clustering algorithm $\C$ under attack. Similarly to \cite{suspicion} the adversary exercises a causative influence by manipulating part of the data to be clustered without any further information about the victim's algorithm $\C$.

Given these considerations, an optimization program for our task is proposed as follows:
\begin{mini}|s|
{\vec{\epsilon} \in \mathcal{E}_{T,\delta}}{\phi(\C(\vec{X}), \C(\vec{X} + \vec{\epsilon}))}
{}{}
\label{opt:2norm}
{}{}
\end{mini}
%\begin{align}
%\min_{\vec{\epsilon} \in %\mathcal{E}_{T,\delta}}{\phi(\C(\vec{X}), %\C(\vec{X} + \vec{\epsilon}))}
%% + \lambda \|\vec{\epsilon}\|_p}
%{}{}
%\label{opt:2norm}
%{}{}
%\end{align}
where $\phi$ is a similarity measure between clusterings, and
\begin{equation}~\label{space}
  \mathcal{E}_{T,\delta} = \{\vec{v} \in \mathbb{R}^{n \times d}, \|\vec{v}\|_{\infty} < \delta \land \vec{v}_i = \vec{0} \ \forall i \notin T\}
\end{equation}
is the \emph{adversarial attack space}, which defines the space of all possible adversarial masks that satisfy the maximum power constraints and perturb only the samples in $T$. A problem without such capability constraints can be denoted with $\mathcal{E}_{\vec{X},\infty}$. Note that $\gamma$ is not directly referenced in $\mathcal{E}_{T,\delta}$ but is bounded by $T$ itself, namely $\gamma = |T| \cdot d$.

We further elaborate Program \ref{opt:2norm} by searching for low Power \& Effort (P\&E) noise masks, in order to enforce the non-detectability of the attack. To this end, we adopt a similar strategy as in \cite{DBLP:conf/iclr/LiuCLS17}, which adds a penalty term $\lambda \|\vec{\epsilon}\|_p$ to the cost function, usually with $p = 0, 2$ or $\infty$.
Following this approach, we reformulate Program \ref{opt:2norm} by including a penalty term that takes into consideration both the attacker's P\&E which leverages the $\infty$ and $0$ norms, respectively. The optimization program becomes:
\begin{mini}|s|
{\vec{\epsilon} \in \mathcal{E}_{T,\delta}}{\phi(\C(\vec{X}), \C(\vec{X} + \vec{\epsilon})) + \lambda \|\vec{\epsilon}\|_0 \|\vec{\epsilon}\|_\infty}
{}{}
\label{opt:zeroInfty}
{}{}
\end{mini}
%\begin{align}
%\min_{\vec{\epsilon} \in %\mathcal{E}_{T,\delta}}{\phi(\C(\vec{X}), %\C(\vec{X} + \vec{\epsilon})) + \lambda %\|\vec{\epsilon}\|_0 \|\vec{\epsilon}\|_\infty}
%{}{}
%\label{opt:zeroInfty}
%{}{}
%\end{align}
This choice keeps the optimization program interpretable since it establishes a straightforward connection to our minimization desiderata (low P\&E).
In addition, our penalty term can be seen as a proxy function for $\|\vec{\epsilon}\|_p$, granting similar regularization properties to the optimization. Indeed the P\&E penalty is an upper bound to the single norm term, as the following lemma shows:
\begin{lemma}\label{lem:normbound}
Let $\vec{x} \in \mathbb{R}^n$, and $p, q \in \mathbb{R} \cup \{+\infty\}$ such that $1 \le p \le q < +\infty$ then:
\begin{equation}
  \|\vec{x}\|_p \leq \|\vec{x}\|_0 \|\vec{x}\|_q
\end{equation}
\end{lemma}
\begin{proof}
The case $\vec{x} = \vec{0}$ is trivial. Suppose that $\forall i$, $x_i \ne 0$, then for a known result on the equivalence of norms in $\mathbb{R}^n$ \cite{10.5555/2422911} we know that $\|\vec{x}\|_p \le n^{(1/p - 1/q)}\|\vec{x}\|_q$, thus:
\begin{align}\label{eq:bounds}
  \|\vec{x}\|_p & \le n^{(1/p - 1/q)}\|\vec{x}\|_q \le n^{1/p}\|\vec{x}\|_q \le n\|\vec{x}\|_q\nonumber\\
                & = \|\vec{x}\|_0\|\vec{x}\|_q
%   \|\vec{x}\|_p & \le n^{1/p}\|\vec{x}\|_q & \implies\nonumber\\
%   \|\vec{x}\|_p & \le \|\vec{x}\|_0^{1/p}\|\vec{x}\|_q & \nonumber
\end{align}
Suppose now, without loss of generality that $\vec{x} = (x_1, \dots, x_m, 0, \dots, 0)^\top$, such that $\forall i \in \{1, \dots, m\}$, $x_i \ne 0$ . Consider its projection $\vec{x}'$ onto the axes $1, \dots, m$, then $\forall p \ge 0$, $\|\vec{x}\|_p = \|\vec{x}'\|_p$. Thus Equation \ref{eq:bounds} holds since:
\begin{equation*}
\|\vec{x}\|_p = \|\vec{x}'\|_p \leq  m\|\vec{x}'\|_q = \|\vec{x}\|_0\|\vec{x}\|_q
\end{equation*}
\end{proof}

\subsection{Threat Algorithm}
The approach we used to optimize Program (\ref{opt:zeroInfty}), takes its inspiration from Genetic Algorithms (GA) \cite{DBLP:books/aw/Goldberg89}. These methods nicely fit our black-box setting since they do not require any particular property on the function to be optimized. Furthermore, our algorithm possesses solid convergence properties. In Section \ref{section:convergence_properties} we show that our algorithm is an instance of the \emph{Abstract Genetic Algorithm} (AGA), as presented in \cite{10.1007/BFb0029725,DBLP:conf/ppsn/EibenAH90}, and we give a proof of its convergence.

An additional constraint, usually imposed in real-world scenarios, is represented by the limited number of queries that can be performed to the algorithm under attack \cite{DBLP:conf/gecco/AlzantotSCZHS19}. Classical approaches in GAs usually create large, fixed-size populations at each generation, and this, in turn, requires to compute the fitness score multiple times, querying $\C$ for each individual in the population, thereby making the process query-inefficient. To address this issue, we propose a growing size population approach. We start with a population $\Theta$ of size equal to $1$ and, generation by generation, we grow it by producing a new individual. To still harness the explorative power of GAs, we use a high mutation rate, and we allow the population set $\Theta$ to grow by keeping trace of all the previously computed individuals. In the case of memory-aware applications, our method can be extended by controlling the size of $\Theta$, in particular, by pruning low-fitness candidates. However, in our experiments, we adopted a different technique aimed to speed up the convergence of the optimization algorithm by reducing the number of generations (cf. Section \ref{section:improvement}).
\begin{algorithm}[H]
 \caption{Black-box poisoning}\label{alg:FCA}
 \begin{algorithmic}[1]
   \STATE {\bfseries Input: $\vec{X} \in \mathbb{R}^{n\times d}, \C, \delta, T, G, l$}
   \STATE \textbf{Output:} optimal adversarial mask $\vec{\epsilon}^*$
   \STATE
   \STATE Initialize $\vec{\epsilon}^{(0)}\in \mathcal{E}_{T,\delta}$ randomly
   \STATE $\Theta = \{\vec{\epsilon}^{(0)}\}$ 

   \STATE \FOR{$g = 0$ {\bfseries to} $G - 1$}
        \STATE $\vec{\epsilon}_{ch}^{(g + 1)} = \texttt{choice}(\Theta, l)$
        \STATE $\vec{\epsilon}_{cr}^{(g + 1)} = \texttt{crossover}(\vec{\epsilon}^{(g)},\vec{\epsilon}_{ch}^{(g + 1)})$
        \STATE $\vec{\epsilon}^{(g + 1)} = \texttt{mutation}(\vec{\epsilon}_{cr}^{(g + 1)}, \delta, T)$
        \STATE $\Theta = \Theta \cup \{\vec{\epsilon}^{(g + 1)}\}$
   \ENDFOR
   \STATE {\bfseries return:} $\vec{\epsilon}^* = \argmin_{\vec{\epsilon} \in \Theta} l(\vec{\epsilon})$
\end{algorithmic}
\end{algorithm}
Algorithm \ref{alg:FCA} describes our optimization approach. It takes as input the feature matrix $\vec{X}$, the clustering algorithm $\C$, the target samples $T$, the maximum attacker's power $\delta$, the total number of generations $G$ (the attacker's budget in term of queries) and the attacker's objective function $l$ (which in our case is the one defined in Program (\ref{opt:zeroInfty})). The resulting output is the optimal adversarial noise mask $\vec{\epsilon}^*$ that minimizes $l$. At each generation, a new adversarial mask $\vec{\epsilon}^{(g + 1)}$ is generated and added to a population set $\Theta$ containing all previous masks.

The core parts of our optimization process are the stochastic operators -- \texttt{choice}, \texttt{crossover} and \texttt{mutation} -- that we use for crafting new candidate solutions with a better fitness score. In the following, we describe their implementation.

\paragraph{Choice} The choice operator is used to decide which candidates will be chosen to generate offspring. We adopt a roulette wheel approach \cite{DBLP:books/aw/Goldberg89}, where only one candidate is selected with a probability proportional to its fitness score, which in turn is \emph{inversely} proportional to the attacker's objective function $l$. Given a candidate $\vec{\epsilon}^{(i)}$, its probability to be chosen for the production process $p(\vec{\epsilon}^{(i)})$ is equal to:
\begin{equation}\label{eq:fitness}
  p(\vec{\epsilon}^{(i)}) = \frac{\exp(-l(\vec{\epsilon}^{(i)}))}{\sum_{\vec{\epsilon} \in \Theta} \exp(-l(\vec{\epsilon}))}
\end{equation}
We remark that our choice operator picks just one adversarial noise mask that is then used in the crossover step.

\paragraph{Crossover} The crossover operator simulates the reproduction phase, by combining different candidate solutions (parents) for generating new ones (offspring). Commonly, crossover operators work with binary-valued strings, however, since our candidates are matrices in $\mathcal{E}_{T,\delta}$, we propose a variant. Given two candidates $\vec{\epsilon}^\prime, \vec{\epsilon}^{\prime\prime} \in \mathcal{E}_{T,\delta}$, the new offspring is generated starting from $\vec{\epsilon}^\prime$, then with probability equal to $p_c$ each entry $i,j$ is swapped with the entry $i,j$ in $\vec{\epsilon}^{\prime\prime}$. The crossover operator has probability $p_c$ of being applied; in the case of failure, $\vec{\epsilon}^\prime$ itself is chosen as an offspring.

\paragraph{Mutation} The mutation is a fundamental operator, usually applied to the offspring generated by the crossover, to introduce genetic variation in the current population. Our operator mutates each entry $\epsilon_{ij}$ s.t. $i \in T$ with probability $p_m$ by adding an uniformly distributed random noise in the range $[-\delta, \delta]$. The resulting perturbation matrix, is subsequently clipped to preserve the constraints dictated by $\mathcal{E}_{T,\delta}$.

Moreover, to enforce the low attacker's effort desiderata, we also perform zero-mutation, meaning that each entry of the mask is set to zero with probability $p_z$.

\subsection*{Time complexity analysis} ~ \label{time_complexity}
In this section we provide a time complexity analysis for Algorithm \ref{alg:FCA}. 
In step 8, the objective function is computed, requiring, in turn, to execute the clustering algorithm $\mathcal{C}$ with complexity $O( \mathcal{C}(nd) )$.
Step 9 performs a cross-over between two adversarial masks, in $O(nd)$ time.
The mutation of Step 10 is similarly computed in $O(nd)$ time. The overall time complexity is, thus, given by $O(G(\mathcal{C}(nd) + 2nd) ) = O(G \mathcal{C}(nd))$, with G equal to the number of generations. 
The complexity of the clustering algorithm $\mathcal{C}(nd)$ is a key point in the efficiency of the attack. As an example, considering $K$-means, we have a polynomial-time of $O(G(ndKt + 2n)) = O(GndKt)$, with $K$ being the number of clusters and $t$ the number of iterations for the clustering algorithm.

\subsection{Speeding up the convergence}~\label{section:improvement}
By just generating a new individual at each generation, our proposed method has the major drawback of being slow at converging. To counter this problem, inspired by the work of \cite{DBLP:journals/corr/GoodfellowSS14}, we decided to ``imprint'' a direction to the generated adversary mask to move the adversarial samples towards the target cluster. Since we lack the gradient information, the centroids information is leveraged instead.
We propose the following approach: each adversarial mask $\vec{\epsilon} \in \mathcal{E}_{T, \delta}$ is generated with the additional constraint that $\forall i,j$ $\epsilon_{ij} \ge 0$. After this, the mask is multiplied by a \emph{direction matrix} $\vec{\psi}$ with $\psi_{ij} = \sgn(c_j^{(t)} - c_j^{(v)})$, $\vec{c}^{(t)}$ and $\vec{c}^{(v)}$ being respectively the target and victim cluster centroids estimated from the victim data. The estimation is performed averaging the samples in the corresponding cluster. This variant can be easily implemented by changing the initialization of $\vec{\epsilon}^{(0)}$ and the mutation step only. It follows that the resulting adversarial attack space is now reduced to $\mathcal{E}_{T,\delta}^\prime \subset \mathcal{E}_{T,\delta}$. We still grant that the capability constraints are respected and the convergence properties hold, although the quality of the found optimum may be inferior.
In addition, we noticed that, without using this strategy,  the optimization algorithm was more sensitive to the choice of hyper-parameters. Therefore, we have decided to adopt this strategy, which makes our algorithm more efficient and less sensitive to the choice of hyper-parameters.

\subsection{Convergence properties}~\label{section:convergence_properties}
In general, GAs do not guarantee any convergence property \cite{DBLP:books/daglib/0017257}. However, under some more restrictive assumptions, it can be shown that they converge to an optimum. In this section, we show that our algorithm can be thought of as an instance of the Abstract Genetic Algorithm (AGA) as presented in \cite{DBLP:conf/ppsn/EibenAH90,10.1007/BFb0029725}. Subsequently, we give a proof of convergence. Below, we show the AGA pseudo-code:
\begin{algorithm}[H]
  \caption{Abstract Genetic Algorithm (pseudo-code)}\label{alg:aga}
  \begin{algorithmic}[1]
    \STATE Make initial population
    \STATE \WHILE{not stopping condition}
      \STATE \textbf{Choose} parents from population
      \STATE Let the selected parents \textbf{Produce} children
      \STATE Extend the population by adding the children to it
      \STATE \textbf{Select} elements of the extended population to survive for the next cycle
    \ENDWHILE
  \STATE \textbf{Output} the optimum of the population
  \end{algorithmic}
\end{algorithm}
In \cite{DBLP:conf/ppsn/EibenAH90,10.1007/BFb0029725}, the authors show that methods such as classical Genetic Algorithms and Simulated Annealing can be thought of as instances of Algorithm \ref{alg:aga}. Further, they prove their probabilistic convergence to a (global) optimum. Following the same theoretical framework, we show that our algorithm indeed satisfies all the conditions for convergence. Before doing so, we first present the framework and adapt our algorithm in order to comply with it.

Let $S$ be a set of candidates and $S^*$ be a set of finite lists over $S$, representing all the possible finite populations. A \emph{neighborhood function} is a function $N: S \rightarrow S^*$ that assigns neighbors to each individual in $S$. A \emph{parent-list}, is a list of candidates able to generate offspring, with $P \subseteq S^*$ denoting the set of all parent-list. In our algorithm, a population is represented by a list $[\vec{\epsilon}^{(0)}, \dots, \vec{\epsilon}^{(g)}]$, therefore $S = \AAS$, $S^* = \AAS^*$, $P = \{[\vec{\epsilon}^{(i)}, \vec{\epsilon}^{(j)}] \mid \vec{\epsilon}^{(i)}, \vec{\epsilon}^{(j)} \in \AAS\}$.

Let $f: X \rightarrow Y$ be a function belonging to $\mathcal{F}$, the set of all functions from $X$ to $Y$. Further, let $(\Omega, \mathcal{A}, \mathbb{P})$ be a probability space and $g: \Omega \rightarrow \mathcal{F}$ be random variable. We define the \emph{randomized} $f$ to be the function $f(\omega, x) = g(\omega)(x)$. Following this definition and \cite{DBLP:conf/ppsn/EibenAH90}, Algorithm \ref{alg:aga} can be then detailed as follows:
\begin{algorithm}[H]%
 \caption{Abstract Genetic Algorithm}\label{alg:agad}
 \begin{algorithmic}[1]
    \STATE Create an $x \in S^*$
    \STATE \WHILE{not stopping condition}
        \STATE draw $\alpha$, $\beta$ and $\gamma$
        \STATE $q = f_c(\alpha, x)$
        \STATE $y = \bigcup_{z \in q} f_p(\beta, z)$
        \STATE $x^\prime = x \cup y$
        \STATE $x = f_s(\gamma, x^\prime)$
    \ENDWHILE
   \STATE output the actual population
\end{algorithmic}
\end{algorithm}%
with $f_c: A \times S^* \rightarrow \mathcal{P}(P)$ being the \emph{choice function}, $f_p: B \times P \rightarrow \mathcal{P}(S)$ being the \emph{production function} and $f_s: C \times S^* \rightarrow S^*$ being the \emph{selection function}. In our case, we define:
\begin{enumerate}[(a)]
  \item $f_c(\alpha, x) = \{[\texttt{choice}(\alpha, x), x_{-1}]\},\ \forall \alpha \in A$
  \item $f_p(\beta, [s_1, s_2]) = \texttt{mutation}(\beta, \texttt{crossover}(\beta, s_1, s_2))),\ \forall \beta \in B$
  \item $f_s(\gamma, x^\prime) = x^\prime$ (Note that our selection is deterministic)
\end{enumerate}
Where $x_{-1}$ is the most recent candidate in the population. In the above pseudo-code, we have explicitly stated the randomization of our procedures $\texttt{choice}$, $\texttt{mutation}$, $\texttt{crossover}$ for clarity. The stochastic processes regulating the drawings of $\alpha$, $\beta$, and $\gamma$ always maintain the same distributions regardless of the current generation, meaning that the probability of generating a new population $x_{new}$ from another one $x_{old}$ does not change over the generations.

We now introduce and extend some definitions presented in \cite{DBLP:conf/ppsn/EibenAH90}:
\begin{enumerate}
  \item A neighborhood structure is \emph{connective} if: $\forall s \in S, \forall t \in S: s \mapsto t$, where $\mapsto$ stands for the transitive closure of the relation $\{(s, t) \in S \times S \mid t \in N(s)\}$.
  \item A choice function is \emph{generous} if: (a) $\{[s, t] \mid s, t \in S\} \subseteq P$ and (b) $\forall x \in S^*,\ \forall s_1, s_2 \in x: \mathbb{P}([s_1, s_2] \in f_c(\alpha, x)) > 0$.
  \item A production function is \emph{generous} if: $\forall s_1, s_2 \in S,\ \forall t \in N(s_1) \cup N(s_2): \mathbb{P}(t \in f_p(\beta, [s_1, s_2])) > 0$.
  \item A selection function is \emph{generous} if: $\forall x \in S^*,\ \forall s \in x: \mathbb{P}(s \in f_s(\gamma, x)) > 0$.
  \item A selection function is \emph{conservative} if: $M_x \cap f_s(\gamma, x) \ne \emptyset$, with $M_x = \{s \in x \mid \forall t \in x: f(s) \le f(t)\}$.
\end{enumerate}
In \cite{DBLP:conf/ppsn/EibenAH90} pag. 10, the authors further make a little technical assumption about the sets $A$, $B$, and $C$, requiring them to be countable, with positive probability for all their members. This is easily achieved in real applications considering the finiteness of the floating point representations.

Now we are ready to prove the following theorem:
\begin{theorem}
    Algorithm \ref{alg:agad} almost surely reaches a global optimum.
\end{theorem}
\begin{proof}
    Given the previous considerations, the following statements hold:
    \begin{enumerate}
        \item \emph{Our neighborhood structure is connective}: by the definition of our mutation operator, it holds that $N(\vec{\epsilon}^{(i)}) = \AAS, \forall \vec{\epsilon}^{(i)} \in \AAS$.
        \item \emph{Our choice function is generous}: this follows from (a) the definition of $P$, and from (b) the positivity of the softmax function in Equation 6.
        \item \emph{Our production function is generous}: See point 1.
        \item \emph{Our selection function is generous}: we allow all the candidates to survive with probability $1$.
        \item \emph{Our selection function is conservative}: see point 4.
    \end{enumerate}
    The proof then follows from Theorem $3$ in \cite{10.1007/BFb0029725}, adjusting the generousness definitions with our versions presented above. The globality of the optimum comes from the fact that our algorithm performs a global search, instead of a local one.
\end{proof}
The same conclusions can be drawn for the speed-up heuristic, by just replacing each instance of $\AAS$, with $\AAS^\prime$.

% \begin{figure}[h]
%   \centering
%   \includegraphics[width=0.35\columnwidth]{img/convergence/fitness_delta03_s05.png}
%   \caption{Convergence profiles in FashionMNIST with $\delta=0.3$ and $s=0.5$.}
%   \label{fig:transf}
% \end{figure}

\section{Experimental results}
In this section, we present an experimental evaluation of the proposed methodology of attack.

\subsection{Robustness analysis}

We ran the experiments on three real-world datasets: FashionMNIST \cite{fashionMNIST}, CIFAR-$10$ \cite{cifar10} and $20$ Newsgroups \cite{DBLP:conf/icml/Lang95}. We focused our analysis on both two- and multiple-way clustering problems. For FashionMNIST and $20$ Newsgroups, we simulated the former scenario in which an attacker wants to perturb samples of one victim cluster $C_v$ towards a target cluster $C_t$. For CIFAR-$10$, we allowed the attacker to move samples from multiple victim clusters towards a target one by simply running multiple times our algorithm with a different victim cluster for each run. In the experiments, we chose $T$ to contain the $s|C_v|$ nearest neighbors belonging to the currently chosen victim cluster, with respect to the centroid of the target cluster. In particular, for FashionMNIST we used $20$ different values for $s$ and $\delta$, in the intervals $[0.01, 0.6]$ and $[0.05, 1]$ respectively; for CIFAR-$10$ we used $20$ different values for $s$ and $\delta$, in the intervals $[0.01, 0.6]$ and $[0.01, 1.5]$ respectively; for $20$ Newsgroups we used $15$ different values for $s$ and $\delta$, in the intervals $[0.01, 0.3]$ and $[0.001, 0.3]$ respectively.
%each dataset, we iterated our analysis for $20$ different values of $s$ evenly spaced in the interval $[0.01,0.6]$. 

We tested the robustness of three standard clustering algorithms: hierarchical clustering using Ward's criterion \cite{ward1963hierarchical}, $K$-means++ \cite{DBLP:conf/soda/ArthurV07} and the normalized spectral clustering \cite{DBLP:journals/pami/ShiM00} as presented in \cite{DBLP:journals/sac/Luxburg07}, with the \cite{DBLP:conf/nips/Zelnik-ManorP04} similarity measure. The code has been written in PyTorch \cite{pytorch} and it available at \footnote{https://github.com/Cinofix/poisoning-clustering}.

For the optimization program, we set $\lambda = \frac{1}{\alpha \cdot n\cdot d}$ with $\alpha = 255$ as penalty term for our cost function. In addition, in the optimization algorithm, we set the probability of having crossover $p_c = 0.85$, mutation $p_m = 0.05$ and zero-mutation $p_z = 0.001$. The total number of generations G, which correspond to the number of queries, was always set to $110$, using the heuristic proposed in Section \ref{section:improvement}.
In addition, we repeated these experiments for five times, reporting the mean with the standard error.

%\paragraph{Choice of similarity measure}\label{section:clustering_comparison}
In Program (\ref{opt:zeroInfty}), we indicate with $\phi$ a function for measuring the similarity between two clustering partitions. In the literature, we can find several metrics used for the evaluation of clusterings \cite{ARI, stgh02b, vscore, AMI};
in addition, \cite{isDataClustering} proposed to adopt the following measure for the evaluation: $d(\vec{Y}, \vec{Y}') = \|\vec{Y}\vec{Y}^\top - \vec{Y}^\prime {\vec{Y}^\prime}^\top\|_F$, where $\|\cdot\|_F$ is the Frobenius norm, and $\vec{Y},\vec{Y}^\prime \in \mathbb{R}^{n \times k}$ are one-hot encodings of the clusterings $\C(\vec{X})$ and $\C(\vec{X} + \vec{\epsilon})$ respectively.
In our work, we decided to use the Adjusted Mutual Information (AMI) Score, proposed in \cite{AMI} since it makes no assumptions about the cluster structure and, as highlighted in \cite{DBLP:journals/jmlr/RomanoVBV16}, it works well even in the presence of unbalanced clusters. Indeed, the clustering partition over the poisoned dataset might also create unbalanced clusters, especially if the attacker wants to move samples only from one towards the others.

The AMI score between two clustering partitions $U$ and $V$ is given by:
\begin{equation}
    AMI(U,V) = {\frac{MI(U, V) - \mathbb{E}[MI(U,V)]}{\max{\{H(U), H(V)\}} - \mathbb{E}[MI(U, V)]}}
\end{equation}
where $MI(U, V)$ measures the mutual information shared by the two partitions, $\mathbb{E}[MI(U,V)]$ represents its expected mutual information and $\max{\{H(U), H(V)\}}$ is the maximum between the two entriopies, which is an upper bound for $MI(U, V)$. AMI is equal to $1$ when the two clustering partitions are identical, and $0$ when they are independent, that is, sharing no information about each other. 

The reader can refer to Section \ref{section:comparison_objective_function} where a comparison analysis between different similarity functions is offered.
\begin{figure*}[t]
  \includegraphics[width=0.32\textwidth]{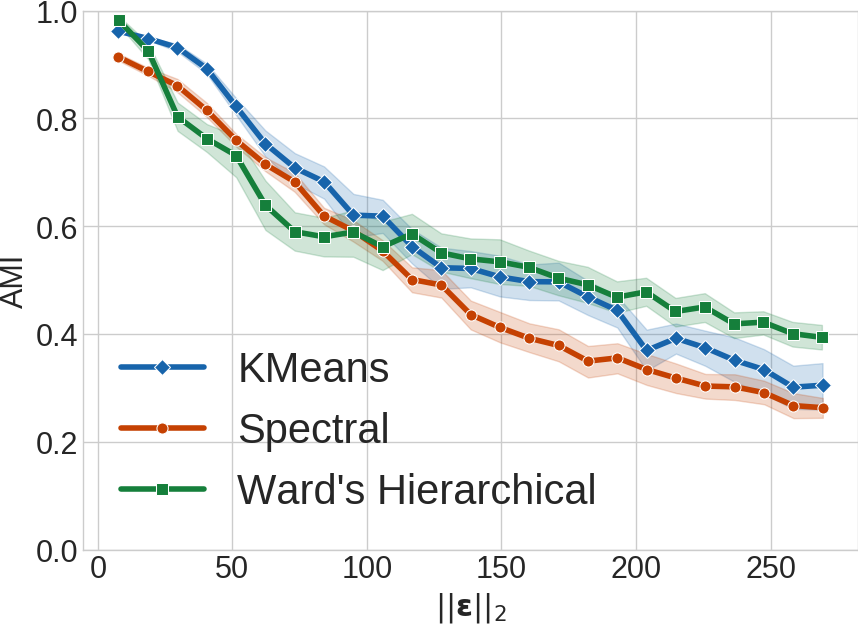}\quad
  \includegraphics[width=0.32\textwidth]{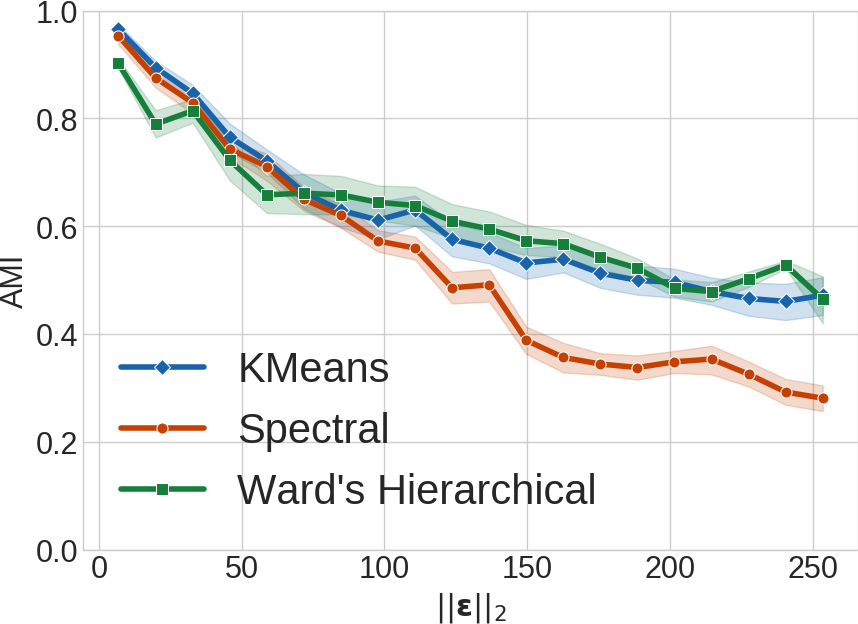}
  \includegraphics[width=0.32\textwidth]{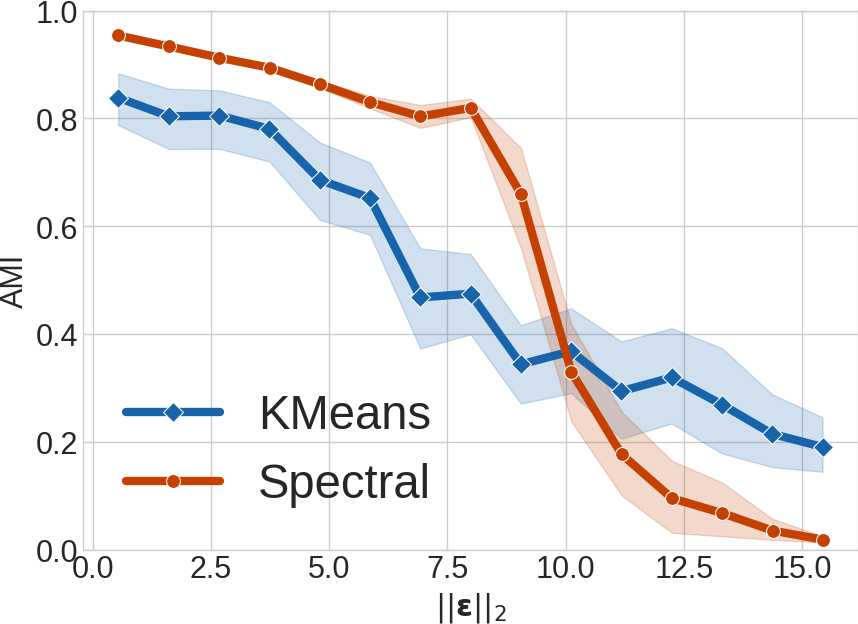}
  \caption{Robustness analysis with FashionMNIST (left), CIFAR-$10$ (middle), $20$ Newsgroups (right). The plots depict the decay of AMI by adversarially perturbing the datasets, with an increasing noise.}
  \label{fig:robustness}
\end{figure*}

\subsubsection{FashionMNIST}~\label{fashionSetting} The FashionMNIST contains $70\,000$ grayscale images of size $28 \times 28$ pixels \cite{fashionMNIST}. In our experiments we randomly sampled $800$ images for class \texttt{Ankle boot} (victim cluster) and $800$ for class \texttt{Shirt} (target cluster). 
In Figure \ref{fig:robustness} (left), we report the obtained results. We observe that the three algorithms have similar behavior and their clustering accuracy consistently decreases with the increment of the adversarial noise level. 
In this case, $K$-means++ shows better performance than spectral clustering, therefore the spectral embedding of data samples seems less robust than raw features only. This fact may suggest that some embedding procedures devised for improving clustering accuracy do not necessarily guarantee robustness against adversarial attacks. However, we reserve further discussion on this in future work.

\subsubsection{CIFAR-10} The CIFAR-$10$ contains $60\,000$ colour images of size $32\times 32$ pixels \cite{cifar10}. We randomly sampled $1\,600$ images from classes \texttt{airplane, frog} and \texttt{automobile}. We addressed the multi-way scenario by first moving samples from \texttt{airplane} and then from \texttt{frog}, always towards the target cluster \texttt{automobile}. Moreover, we used a ResNet50 for features extraction, and we performed clustering on the resulting feature space, obtaining better initial results in terms of AMI. In Figure \ref{fig:robustness} (middle), we show the performance of the three clustering algorithms under adversarial manipulations.
We observe that our attacks significantly decrease the clustering quality for the three algorithms. Even if the ResNet50 features allow cluster algorithms to achieve better performance, they are still vulnerable to adversarial noise.
Further, note how the gap in performance of spectral clustering and $K$-means++ has even increased when adopting a DNN-generated embedding.

We provide, in Figure \ref{fig:cifar_viz}, a visual representation of poisoning samples for CIFAR-10. We reconstructed the poisoning samples from the feature space using the feature collision strategy adopted in \cite{DBLP:conf/nips/ShafahiHNSSDG18}, where the target is exactly our poisoning sample.

\subsubsection{20 Newsgroups} The $20$ Newsgroups is a dataset commonly used for text classification and clustering, which contains $20\,000$  newspaper articles divided into $20$ categories. The experiments were conducted with two highly unrelated categories of news, \texttt{rec.sport.baseball} (victim cluster) and \texttt{talk.politics.guns} (target cluster). We applied the a combination of TF-IDF \cite{DBLP:journals/ipm/SaltonB88} and LSA \cite{doi:10.1080/01638539809545028} to embed features into a lower dimensional space. The resulting feature matrix had dimension $1\,400 \times 80$. In this case, we tested our method against two ensemble clustering algorithms, derived from $K$-means and spectral clustering algorithms (hierarchical clustering was not giving good enough clustering performance). The two algorithms use the Silhouette value \cite{ROUSSEEUW198753}, and the clustering with the maximum silhouette score is selected as the best one. In particular, we ran $20$ instances of the $K$-means algorithm with random centroids initializations, while, for spectral clustering, we ran $3$ instances of the algorithm proposed in \cite{DBLP:journals/sac/Luxburg07} with $3$ different similarity measures. Given a sample pair $\vec{x}_i$ and $\vec{x}_j$, the measures are:
\begin{align}
    s_{ij} & = \frac{\vec{x}_i^\top \vec{x}_j}{\|\vec{x}_i\|_2 \|\vec{x}_j\|_2}\label{cosine}\\
    s_{ij} & = \frac{(\vec{x}_i - \bar{\vec{x}})^\top
                (\vec{x}_j - \bar{\vec{x}})}
                {\|\vec{x}_i  - \bar{\vec{x}}\|_2 \|\vec{x}_j  - \bar{\vec{x}}\|_2}\label{corr}\\
    s_{ij} & = d_{max} - \|\vec{x}_i - \vec{x}_j\|_2\label{dist_max}
\end{align}
Equation \ref{cosine} represents the cosine similarity between two samples $\vec{x}_i$ and $\vec{x}_j$. Equation \ref{corr} is the Pearson correlation coefficient, with $\bar{\vec{x}}$ being the sample mean. Moreover, we introduced a sparsification technique, clamping to $0$ all negative values, which improved the clustering performance. Finally, in Equation \ref{dist_max} we define $d_{max} = \max_{ij} \|\vec{x}_i - \vec{x}_j\|_2$.\\

Figure \ref{fig:robustness} (right) reports the performance of two clustering algorithms under adversarial manipulation. Ensemble methods are known to be more robust against random noise with respect to the normal behavior of the corresponding algorithms \cite{DBLP:journals/jair/OpitzM99, DBLP:conf/icdm/NguyenC07}; however, our attacking model was able to fool them and significantly decreased their clustering performance. In this case, in low noise regime, spectral clustering seems to benefit the ensembling technique, however its behavior follows previous experiments.

\begin{figure}[!t]
  \centering
  \includegraphics[width=0.42\columnwidth]{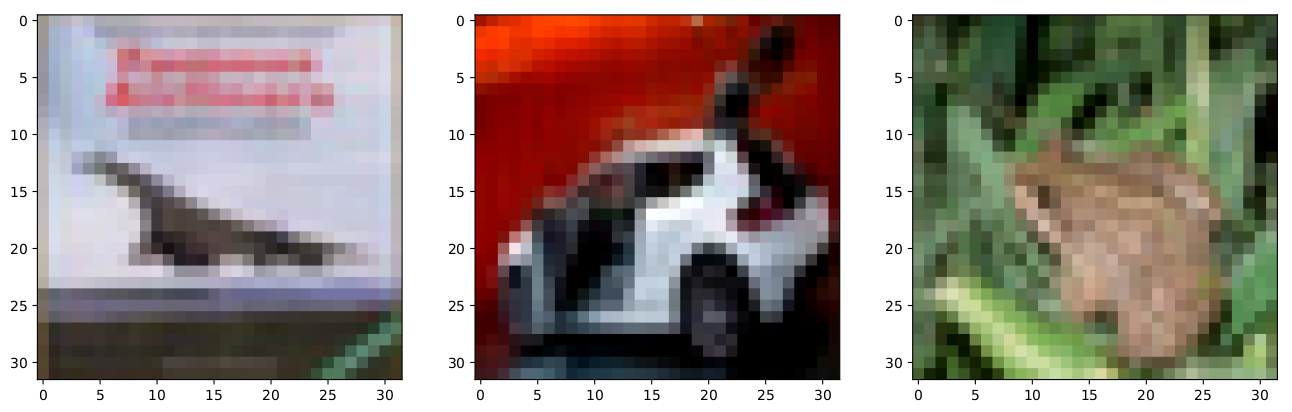}
  
  \includegraphics[width=0.42\columnwidth]{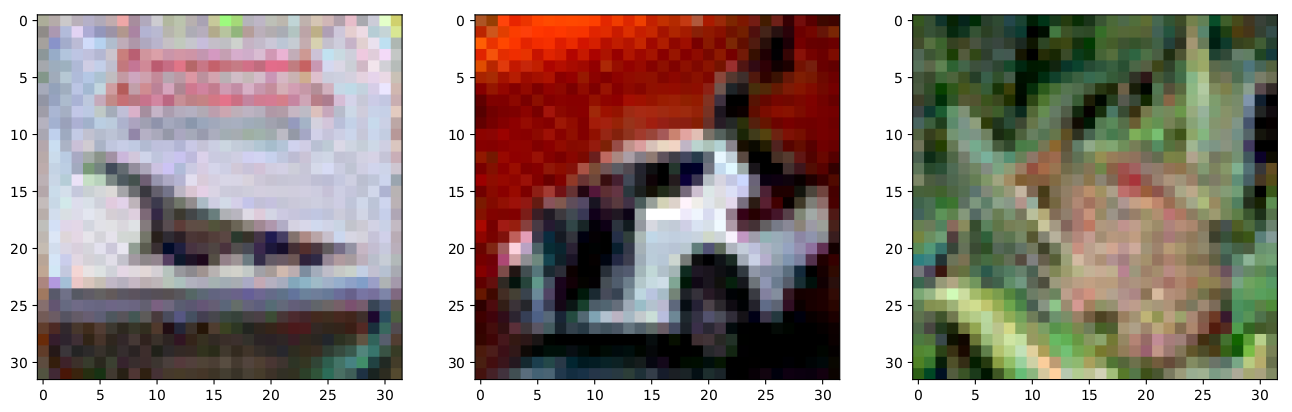}
  \caption{(Top row) clean samples from CIFAR-10. (Bottom row) the corresponding poisoning samples with $\delta$ = $0.1$.}
  \label{fig:cifar_viz}
\end{figure}

\subsubsection{Choice of the cost function}\label{section:comparison_objective_function}
In our optimization program, we employed the AMI both for the cost function $\phi$ and the evaluation measure. Indeed, our optimization being not derivative-dependent, we could adopt the same function for the optimization and evaluation tasks.
\begin{figure}[!t]
  \centering
  \includegraphics[width=0.42\columnwidth]{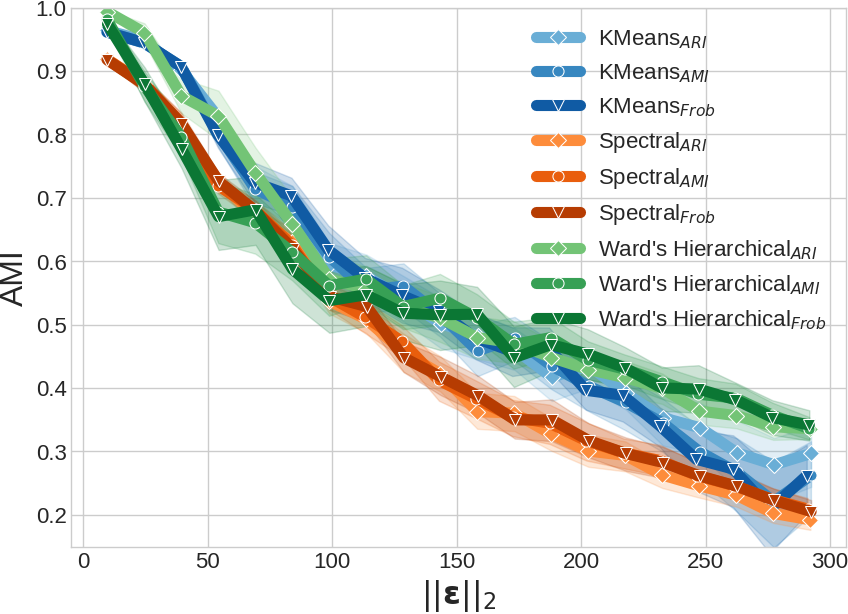}\quad
  \includegraphics[width=0.42\columnwidth]{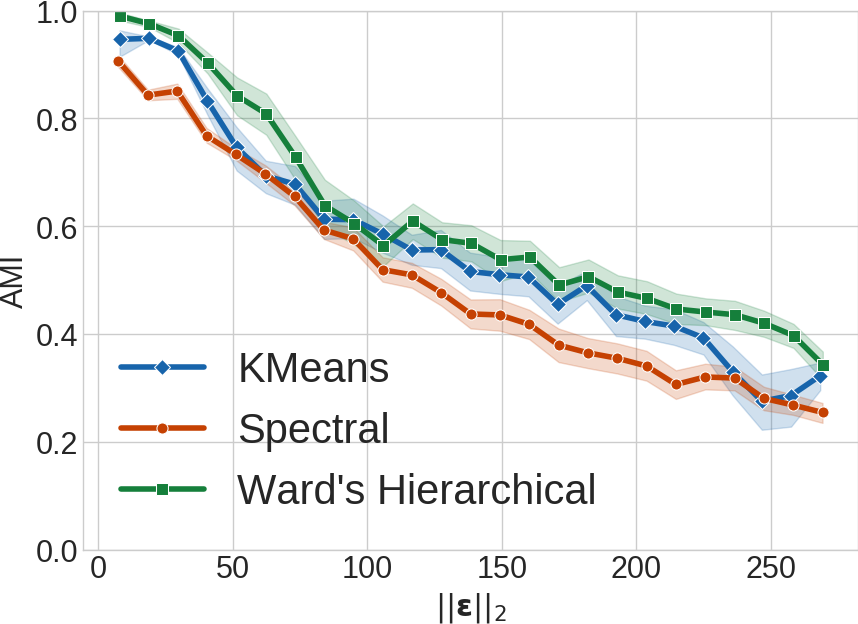}
  \caption{Robustness analysis on FashionMNIST by changing the similarity measure $\phi$ (left), and crafting the poisoning samples employing a surrogate dataset (right).}
  \label{fig:ablation}
\end{figure}
We decided to analyze the impact of the clustering similarity function $\phi$, with FashionMNIST, and see if there were significant differences among each other. In Figure \ref{fig:ablation} (left), shows the variation of results with $\phi$ equals to ARI \cite{ARI}, AMI \cite{AMI} and the negated distance proposed in \cite{isDataClustering} (referred as ``Frob''). The plot shows no substantial difference among the choices of $\phi$, suggesting that this hyperparameter does not significantly impact the optimization process.

\subsubsection{Surrogate data}
We ran additional experiments on a more challenging scenario by relaxing the knowledge assumption of the target data. In this scenario, the attacker does not have access to the target data and can only sample a surrogate dataset from the same distribution. To this end, we have randomly sampled two subsets of 1600 images, each extracted from FashionMNIST as detailed in Section \ref{fashionSetting}. We use the first one to create the poisoning samples and then evaluate their effectiveness on the other one.  Figure  \ref{fig:ablation}  (right) shows that our attack is strong enough to decrease the clustering performance even when the attacker has no access to the target dataset.

\subsection{Comparison}  ~ \label{sec:comparison}
To the best of our knowledge, the only work dealing with adversarial clustering in a black-box way is \cite{suspicion}. In this work, the authors presented a new type of attack called \emph{spill-over}, in which the attacker wants to assign as many samples as possible to a wrong cluster by poisoning just one of them. They proposed a threat model against linearly separable clusters to generate such kind of adversarial noise.

\begin{table}[t]
\renewcommand{\arraystretch}{0.7}
\centering
\begin{small}
    \begin{sc}
    \begin{tabular}{lcccr}
      \toprule
      Method & $\|\vec{\epsilon}\|_0$ & $\|\vec{\epsilon}\|_2$ & $\|\vec{\epsilon}\|_\infty$ & \#Miss-clust \\
      \midrule
      Spill-over            &   413 &  872.8 & 146.8 & $2$\\
      Spill-over$_{clamp}$  &   412 &  828.2 & 146.8 & $2$\\

      %Ours $(\delta = 73.43)$ & $138 \pm 14.4$ & $523.6 \pm 32.4$ & $73.4 \pm 1.4$ & $12.0 \pm 0.0$\\
      %Ours $(\delta = 146.87)$ & $33 \pm 14.3$ & $497.9 \pm 92.5$ & $146.0 \pm 1.4$ & $12.2 \pm 1.0$\\
      %Ours $(\delta = \infty)$ & $24 \pm 10.7$ & $696.2 \pm 204.1$ & $243.4\pm 13.6$ & $13.7 \pm 2.2$\\
      
      Ours $(\delta = 73.43)$ & $151 \pm 19.2$ & $551.9 \pm 36.3$ & $73.4 \pm 1.4$ & $12.0 \pm 0.0$\\
      Ours $(\delta = 146.87)$ & $30 \pm 7.6$ & $479.1 \pm 62.5$ & $145.0 \pm 3.7$ & $12.0 \pm 0.0$\\
      Ours $(\delta = \infty)$ & $29 \pm 13.8$ & $757.1 \pm 203.2$ & $246.25\pm 16.8$ & $14.3 \pm 2.4$\\
      \bottomrule
    \end{tabular}
      \end{sc}
\end{small}
\caption{Comparison on MNIST with digits 3\&2.}
\label{tab:mnist32_comp}
\end{table}

\begin{table}[t]
\renewcommand{\arraystretch}{0.7}
\centering
\begin{small}
    \begin{sc}
    \begin{tabular}{lcccr}
      \toprule
      Method & $\|\vec{\epsilon}\|_0$ & $\|\vec{\epsilon}\|_2$ & $\|\vec{\epsilon}\|_\infty$ & \#Miss-clust  \\
      \midrule
      Spill-over            &   152 &    585.3 & 159.7 &  11\\
      Spill-over$_{clamp}$  &   151 &  463.2 & 131.7 & 9\\

      Ours $(\delta=79.86)$ & $117 \pm 7.8$ & $528.4 \pm 30.1$& $79.8 \pm 2.8$& $9.1 \pm 0.4$\\
      
      Ours $(\delta=159.72)$ & $75 \pm 22.4$ & $782.7 \pm 124.2$& $ 159.3 \pm 1.3$ & $12.0\pm 4.5$\\
      
      Ours $(\delta = \infty)$ & $46 \pm 19.4$ & $902.7 \pm 205$ & $248.3 \pm 8.8$& $14.6 \pm 4.5$\\
      \bottomrule
    \end{tabular}
  \end{sc}
\end{small}
\caption{Comparison on MNIST with digits 4\&1.}
\label{tab:mnist41_comp}
\end{table}
To have a fair comparison, we performed \emph{spill-over} attacks on the same settings of the aforementioned work, comparing the performance on MNIST and UCI Handwritten Digits datasets\footnote{A dataset containing $5\,620$ grayscale images of size $8 \times 8$, with intensities in the range $[0, 16]$} \cite{alpaydin1995a}, targeting Ward's hierarchical clustering algorithm. Further details can be found in Appendix.

For MNIST, we considered the digit pairs $4\&1$ and $3\&2$, while for Digits, we considered the digit pairs $4\&1$, and $8\&9$. Our algorithm was run with the $\delta = \Delta$ which is the maximum acceptable noise threshold found by the authors, with $\delta = \Delta / 2$ and with $\delta = \infty$. We found the value of $\Delta$ used in \cite{suspicion} by looking at the source code.
We imposed to attack just one sample ($|T| = 1$), namely the nearest neighbor to the centroid of the target cluster. We performed our experiments $20$ times, reporting mean and std values.
\begin{table}[!t]
\renewcommand{\arraystretch}{0.7}
\centering
\begin{small}
  \begin{sc}
    \begin{tabular}{lcccr}
      \toprule
      Method & $\|\vec{\epsilon}\|_0$ & $\|\vec{\epsilon}\|_2$ & $\|\vec{\epsilon}\|_\infty$ & \#Miss-clust  \\
      \midrule
      Spill-over & 54 & 15.70 & 9.44 & 21\\
      Spill-over$_{clamp}$  &   54 &  15.70 & 9.44 & 21\\

      Ours $(\delta = 4.72) $ & $12 \pm 1.20$ & $11.49 \pm 1.25$ & $4.7 \pm 0$ & $21 \pm 0.0$\\
      Ours $(\delta = 9.44)$ & $7 \pm2.85$ & $13.86 \pm 2.96$ & $8.12 \pm 1.24$ & $21 \pm 0.0$\\
      Ours $(\delta = \infty)$ & $4 \pm 1.74$ & $15.18\pm 3.16$ & $10.94 \pm 1.49$ & $21 \pm 0.0$\\
      \bottomrule
    \end{tabular}
      \end{sc}
\end{small}
\caption{Comparison on Digits with digits 8\&9.}
\label{tab:digit89_comp}
\end{table}
\begin{table}[!t]
\renewcommand{\arraystretch}{0.7}
\centering
\begin{small}
  \begin{sc}
    \begin{tabular}{lcccr}
      \toprule
      Method & $\|\vec{\epsilon}\|_0$ & $\|\vec{\epsilon}\|_2$ & $\|\vec{\epsilon}\|_\infty$ & \#Miss-clust  \\
      \midrule
      Spill-over & 14 & 23.93 & 11.89 & 24\\
      Spill-over$_{clamp}$  &  11 &  16.28 & 9.93 & 21\\

      Ours $(\delta = 5.94)$ & $13 \pm 1.70$ & $16.27 \pm 1.20$ & $5.94 \pm 0.0$& $24 \pm 0.0$\\
      Ours $(\delta = 11.89)$ & $7 \pm 2.03$ & $19.84 \pm 1.96$ & $11.13 \pm 0.79$ & $24 \pm 0.0$\\
      Ours $(\delta = \infty)$ & $7 \pm 2.36$ & $21.06 \pm 2.36$ & $12.79 \pm 4.34$ & $24 \pm 0.0$\\
      \bottomrule
    \end{tabular}
  \end{sc}
\end{small}
\caption{Comparison on Digits with digits 4\&1.}
\label{tab:digit41_comp}
\end{table}
The results are presented in Table \ref{tab:mnist32_comp}-\ref{tab:digit41_comp} along with more details on the experiments. Although our algorithm achieves its best performance by moving more samples at once, we were able to match, or even exceed, the number of spill-over samples (\texttt{\#Miss-clust}) achieved in \cite{suspicion}, even when halving the attacker's maximum power proposed by the authors. Moreover, the results show also that we were able to craft adversarial noise masks $\vec{\epsilon}$, which were significantly less detectable in terms of $\ell_0, \ell_\infty$.

In Table \ref{tab:comp_seed}-\ref{tab:comp_mocap}, we report a comparison for the $K$-means++ algorithm with UCI Wheat Seeds \cite{10.1007/978-3-642-13105-9_2} and MoCap Hand Postures \cite{DBLP:conf/cvpr/GardnerKDS14} dataset, repeating the same experimental setting of \cite{suspicion}.  We obtain the same number of spill-over samples (\texttt{\#Miss-Clust}) with significant lower Power \& Effort.

Finally, in Figure \ref{suspicion_mnist32} and \ref{suspicion_digits14}, we show a qualitative assessment of the crafted adversarial spill-over samples. Note that the crafted adversarial examples of \cite{suspicion} do not preserve box-constraints commonly adopted for image data. Indeed, pixel intensities exceed $255$ and $16$ for MNIST and Digits, respectively. We also evaluated the performance of \cite{suspicion} by clamping the resulting adversarial examples (\texttt{Spill-over$_{clamp}$}), and we observe that the number of spill-over samples is reduced.
\begin{table}[t]
\renewcommand{\arraystretch}{0.7}
\centering
\begin{small}
    \begin{sc}
    \setlength\tabcolsep{3.pt}
    \begin{tabular}{lcccr}
      \toprule
      Method & $\|\vec{\epsilon}\|_0$ & $\|\vec{\epsilon}\|_2$ & $\|\vec{\epsilon}\|_\infty$ & \#Miss-clust \\
      \midrule
      Spill-over  &   7 &  0.42 & 0.30 & $2$\\
      Ours $(\delta = 0.15)$  & $3 \pm 0.79$ & $0.14 \pm 0.04$ & $0.10 \pm 0.03$ & $2.0 \pm 0.0$\\
      Ours $(\delta = 0.30)$ & $3 \pm 0.76$ & $0.28 \pm 0.09$ & $0.21 \pm 0.06$ & $2.0 \pm 0.0$\\
      \bottomrule
    \end{tabular}
  \end{sc}
\end{small}
\caption{Comparison for Seeds.}
\label{tab:comp_seed}
\end{table}
\begin{table}[t]
\renewcommand{\arraystretch}{0.7}
\centering
\begin{small}
  \begin{sc}
    \setlength\tabcolsep{3.pt}
    \begin{tabular}{lcccr}
      \toprule
      Method & $\|\vec{\epsilon}\|_0$ & $\|\vec{\epsilon}\|_2$ & $\|\vec{\epsilon}\|_\infty$ & \#Miss-clust  \\
      \midrule
      Spill-over  &   9 &  44.42 & 20.0 & $5$\\
      %Ours $(\delta = 10)$ & $1 \pm 0.45$ & $5.87 \pm 2.38$ & $5.63 \pm 2.29$ & $5.0 \pm 0.0$\\
      %Ours $(\delta = 20)$ & $1 \pm 1.36$ & $13.66 \pm 9.51$ & $11.39 \pm 6.06$ & $5.0 \pm 0.0$\\
      
      Ours $(\delta = 10)$ & $1 \pm 0.48$ & $5.13 \pm 1.86$ & $5.0 \pm 1.86$ & $5.0 \pm 0.0$\\
      Ours $(\delta = 20)$ & $1 \pm 1.14$ & $8.50 \pm 6.61$ & $7.74 \pm 5.18$ & $5.0 \pm 0.0$\\
      \bottomrule
    \end{tabular}
  \end{sc}
\end{small}
\caption{Comparison for MoCap Hand Postures.}
\label{tab:comp_mocap}
\end{table}

\begin{figure}[!t]
\centering
  \includegraphics[width=0.55\columnwidth]{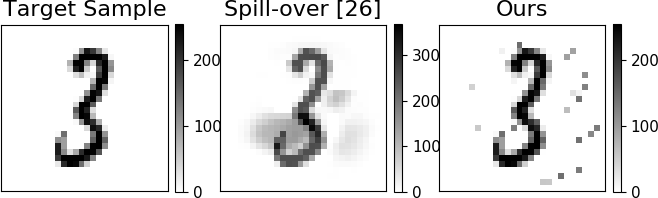}%
  \caption{Spill-over samples for MNIST. The target sample (left), the corresponding adversarial sample crafted with the attack proposed in \cite{suspicion} (middle), and our adversarial sample with
  $\delta=146.87$} (right).
  \label{suspicion_mnist32}%
\end{figure}
\begin{figure}[!t]
\centering
  \includegraphics[width=0.55\columnwidth]{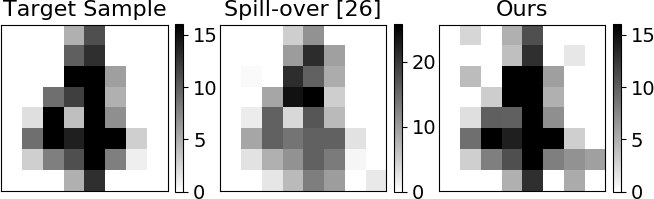}%
  \caption{Spill-over samples for Digits. The target sample (left), the corresponding adversarial sample crafted with the attack proposed in \cite{suspicion} (middle), and our adversarial sample with $\delta=11.89$ (right).} 
  \label{suspicion_digits14}
\end{figure}

In conclusion, \cite{suspicion} aims to find $\Delta$, which does not lead to the attack being considered an outlier using the , Coordinate-wise Min-Mahalanobis-Depth (COMD) measure, at the expense of existing box-constraints. Whereas our purpose consists of proposing an algorithm that can effectively corrupt a black-box clustering algorithm's performance by minimizing the attacker's power and effort (P\&E). Indeed, our attacks, as shown in Table \ref{tab:digit89_comp}-\ref{tab:comp_mocap}, show lower $\ell_0$ and $\ell_\infty$ compared to the attack obtained with \cite{suspicion}. These results suggest that our algorithm can craft effective poisoning attacks, even stronger than \cite{suspicion}, with less P\&E and satisfying the box-constraints

\subsubsection*{Ablation study}
We provide an ablation study for the mutation and zero rate parameters,  $p_m$ and $p_z$ respectively, keeping the crossover rate $p_c$ set to $0.85$ for the two pairs of MNIST digits. The high crossover rate is a common choice in Genetic Algorithms \cite{why_high_cr_ahmad}. We use the same setting described in Section \ref{sec:comparison} with dataset MNIST 3\&2 and MNIST 4\&1.

\begin{figure*}[t]
\centering
  \includegraphics[width=1\textwidth]{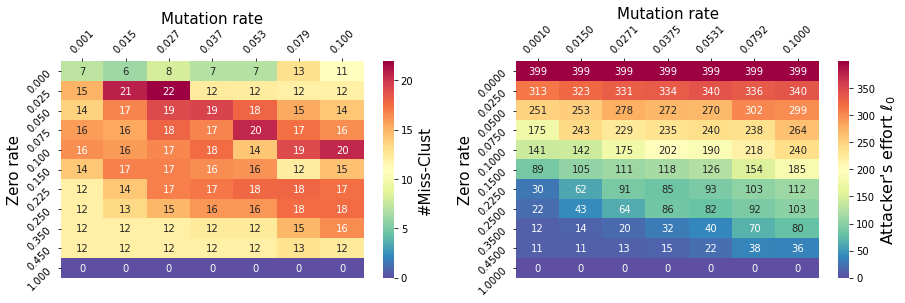}

  \caption{Ablation study for mutation rate (x-axis) and zero rate (y-axis) for MNIST 3\&2. (left) Number of miss clustered samples, (right) attacker’s effort, i.e., number of manipulated pixels.}
  \label{fig:mnist32_ablation}
\end{figure*}
\begin{figure*}[t]
\centering
  \includegraphics[width=1.\textwidth]{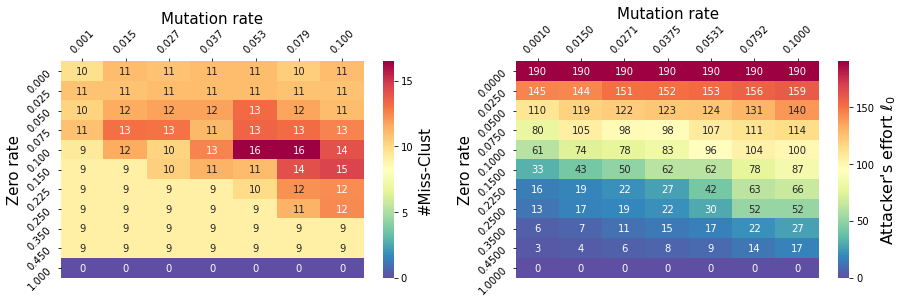}
  \caption{Ablation study for mutation rate (x-axis) and zero rate (y-axis) for MNIST 4\&1. (left) number of miss clustered samples, (right) attacker’s effort, i.e., number of manipulated pixels.}
  \label{fig:mnist41_ablation}
\end{figure*}

Fig. \ref{fig:mnist32_ablation} and Fig. \ref{fig:mnist41_ablation} reveal how the two hyperparameters affect the attacker's effort and the strength of the attacks created. An increment of the zero rate implies attacks with less effort, while an increment of the mutation inverts this tendency thus generating more powerful attacks.
From Figure \ref{fig:mnist32_ablation}-\ref{fig:mnist41_ablation} (left), we observe that the yellow regions at the bottom left determine a good compromise on having a small effort and good attack effectiveness. However, by increasing the attacker's effort, we can generate even more effective attacks. The results of the ablation study with the two pairs of digits are similar, suggesting that the same set of hyperparameters can be chosen without significant differences, as also suggested by the wide bottom-left regions of Figure \ref{fig:mnist32_ablation}-\ref{fig:mnist41_ablation} (left) where the number of miss-clustered observations is constant. This behavior suggests that we do not really need an extensive hyperparameters tuning procedure to obtain effective poisoning attacks.

The results obtained against \cite{suspicion} use a combination of hyperparameters that allows the attacker's effort to be kept low while still maintaining outstanding comparison results. However, a better choice of hyperparameters would have allowed us to further improve our results in terms of miss-clustered points and the attacker's effort.

\subsection{Empirical convergence}
In addition to the theoretical results on convergence provided in Section \ref{section:convergence_properties}, we propose an empirical analysis of convergence. In particular, for a pre-set configuration of $\delta$ and $s$, we performed a series of attacks on the FashionMNIST dataset, with an increasing number of generations/queries, evaluating the trend of objective function presented Program (\ref{opt:zeroInfty}). The results are reported in Figure \ref{fig:convergence}. It can be seen that our algorithm requires a relatively low number of queries to converge to a minimum, with the exception of $K$-means++ that presents a slower convergence than spectral clustering, most probably due to the nature of the feature embeddings used.

\subsection{Transferability}
In \cite{transferability}, the authors defined the concept of \emph{transferability} of adversarial examples. In particular, an adversarial example generated to mislead a model $f$ is said to be transferable if it can mislead other models $f^\prime$. It was further observed that if the attacker has limited knowledge about the model under attack, she may train a substitute model, craft adversarial samples against it, and then devise them to fool the target model.

The authors analyzed this property only between classification algorithms. We extend this analysis showing that even adversarial samples crafted against clustering algorithms are suitable and can be transferred to fool supervised models successfully. To the best of our knowledge, this is the first work that proposes an analysis of transferability between unsupervised and supervised algorithms.
\begin{figure}[!t]
\centering
  \includegraphics[width=0.42\columnwidth]{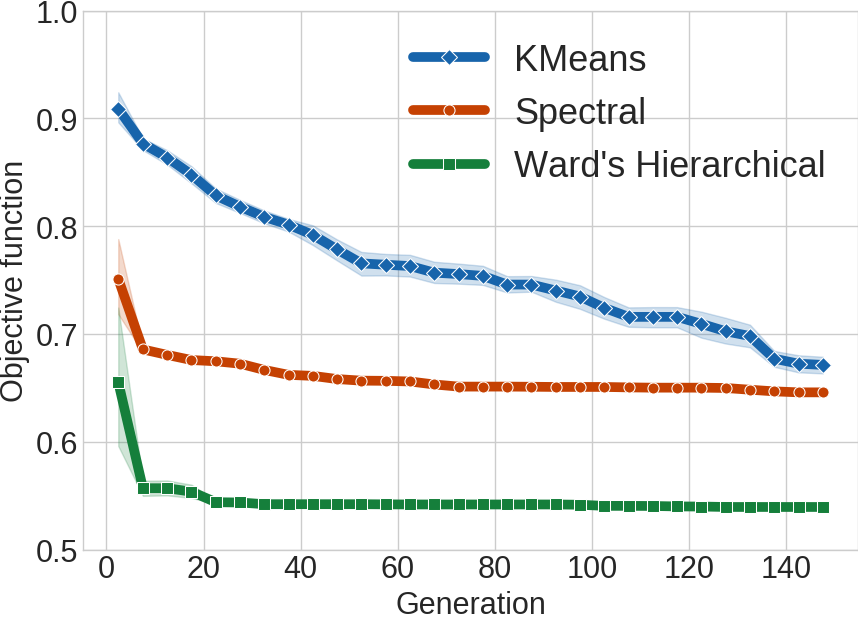}%
  \caption{Convergence of objective function on FashionMNIST. $\delta = 0.2$, $s = 0.25$.}%
  \label{fig:convergence}
\end{figure}
\begin{figure}[!t]
  \centering
  \includegraphics[width=0.42\columnwidth]{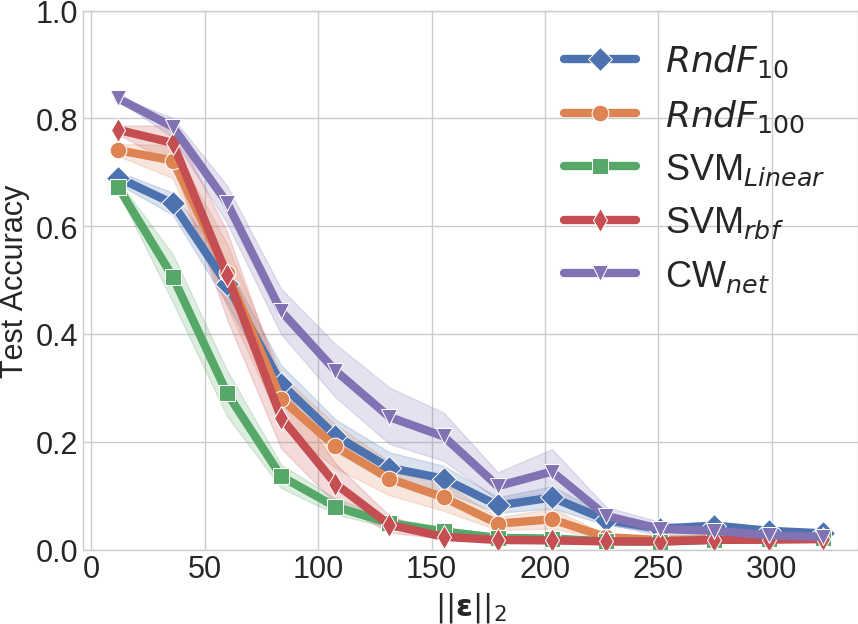}
  \caption{Transferability of adversarial examples against supervised classification models. By progressively incrementing the P\&E of the perturbation matrix, our algorithm crafts adversarial samples that effectively decrease the performance of the algorithm evaluated on the two target classes.}
  \label{fig:transf}
\end{figure}
We evaluated the transferability properties of our noise by attacking the $K$-means++ algorithm on $2\,000$ testing samples taken from labels FashionMNIST (\texttt{Ankle boot}, \texttt{Shirt}). In particular, we used the crafted adversarial samples to test the robustness of several classification models, trained on $60,000$ samples of FashionMNIST. The tested model are a linear and an RBF SVM \cite{DBLP:journals/datamine/Burges98}, two random forest \cite{DBLP:journals/ml/Breiman01} with $10$ and $100$ trees respectively\footnote{For both the SVM and the Forest, we used the implementation proposed in the scikit-learn \cite{pedregosa2011scikit} library.}, and the Carlini \& Wagner (C\&W) deep net proposed in \cite{CandW}, following the same training setting. Table \ref{table:accuracy} shows the test accuracy on the full dataset and only for classes \texttt{Ankle boot} and \texttt{Shirt}.
\begin{table}[!t]
\begin{footnotesize}
\centering
  \begin{sc}
    \setlength\tabcolsep{5.5pt}
  \renewcommand{\arraystretch}{0.8}
    \centering
    \begin{tabular}{lcc}
      \toprule
      Model & All classes & Two classes\\
      \midrule
      Linear SVM & 0.846 & 0.753\\
      RBF SVM & 0.882 & 0.8025\\
      R. Forest 10 & 0.856 & 0.739\\
      R. Forest 100 & 0.875 & 0.769\\
      C\&W net & 0.915 & 0.866\\
      \bottomrule
    \end{tabular}
  \end{sc}
  \caption{Test accuracies on FashionMNIST test dataset, for the chosen classifiers. The middle column presents the overall performance of the trained models. The right column presents the accuracy evaluated on the two target classes (\texttt{Ankle boot} and \texttt{Shirt}).}
  \label{table:accuracy}
\end{footnotesize}
\end{table}
We report the results in Figure \ref{fig:transf}, where the accuracy over the poisoned samples only is reported. The results show clear evidence on the transferability of our adversarial noise, crafted against $K$-means++, to the tested classifiers. Note further that the C\&W net is the most accurate model, performing slightly better than the RBF SVM, and the latter seems to be more sensitive to our adversarial noise.

\section{Conclusions and Future Work}
In this work, we have proposed a new black-box, derivative-free adversarial methodology to fool clustering algorithms and an optimization method inspired by genetic algorithms, able to find an optimal adversarial mask efficiently. We have conducted several experiments to test the robustness of classical single and ensemble clustering algorithms on different datasets, showing that they are vulnerable to our crafted adversarial noise. We have further compared our method with a state-of-the-art black-box adversarial attack strategy, showing that we outperform its results both in terms of attacker's capability requirements and misclustering error. Finally, we have also seen that the crafted adversarial noise can be applied successfully to fool supervised algorithms too, introducing a new transferability property between clustering and classification models.

In our work, we have brought attention to many possible topics of research, which we summarize in the following. First of all, since our proposed method can be easily adaptable to more challenging problems, we plan to address the evasion problem on supervised models. Furthermore, to better characterize the robustness against different kinds of datasets, we plan to analyze the relationship between the sparsity of the data and the robustness of the clustering algorithms. Finally, as the work considers only the clustering problem, our analysis can be extended to different unsupervised tasks, such as unsupervised image segmentation, widespread in sensible applications as well.

\bibliography{mybibfile}

\section*{Appendix} ~\label{appendix}

%\subsection{Additional details on the experiments}
For the sake of completeness and reproducibility of the experimental setting, in the following, we report a detailed list of all hyperparameters used in our experiments. In particular Table \ref{tab:digits_comp_par} and \ref{tab:mnist_comp_par} present all the hyperparameters used in our method for the comparison with \cite{suspicion}.
\begin{table}[h]
    \centering
    \begin{scriptsize}
        \begin{sc}
            \setlength\tabcolsep{3.0pt}
              \renewcommand{\arraystretch}{0.8}
            \begin{tabular}{lccccr}
              \toprule
              Method & G & $\lambda$ & $p_c$ & $p_m$ & $p_z$\\
              \midrule
              Ours $(\delta = 73.43)$ & 150 & 1& 0.85 & 0.2 & 0.35\\
              Ours $(\delta = 146.87)$ & 150 & 1& 0.85 & 0.01 & 0.20\\
              Ours $(\delta = \infty)$ & 150 & 1& 0.85 & 0.001 & 0.25\\
              \bottomrule
            \end{tabular}
        \end{sc}
            % \caption{Comparison parameters for Digits dataset with digits 8\&9.}
            % \label{tab:digits_comp_par89}
    \end{scriptsize}
    \begin{scriptsize}
        \begin{sc}
            \setlength\tabcolsep{3.0pt}
              \renewcommand{\arraystretch}{0.8}
            \begin{tabular}{lccccr}
              \toprule
              Method & G & $\lambda$ & $p_c$ & $p_m$ & $p_z$\\
              \midrule
              Ours $(\delta = 73.43)$ & 150 & 1& 0.85 & 0.2 & 0.35\\
              Ours $(\delta = 146.87)$ & 150 & 1& 0.85 & 0.01 & 0.20\\
              Ours $(\delta = \infty)$ & 150 & 1& 0.85 & 0.001 & 0.25\\
             \bottomrule
            \end{tabular}
      \end{sc}
  \caption{Comparison parameters for Digits dataset with digits 8\&9 (left) and 4\&1 (right).}
  \label{tab:digits_comp_par}
\end{scriptsize}
\end{table}

\begin{table}[h]
\centering
    \begin{scriptsize}
        \begin{sc}
            \setlength\tabcolsep{3.0pt}
          \renewcommand{\arraystretch}{0.8}
            \begin{tabular}{lccccr}
              \toprule
              Method & G & $\lambda$ & $p_c$ & $p_m$ & $p_z$\\
              \midrule
              Ours $(\delta = 73.43)$ & 150 & 1& 0.85 & 0.015 & 0.10\\
              Ours $(\delta = 146.87)$ & 150 & 1& 0.85 & 0.015 & 0.25\\
              Ours $(\delta = \infty)$ & 150 & 1& 0.85 & 0.005 & 0.25\\
              \bottomrule
            \end{tabular}
        \end{sc}
    \end{scriptsize}
    \begin{scriptsize}
        \begin{sc}
            \setlength\tabcolsep{3.0pt}                  \renewcommand{\arraystretch}{0.8}
            \begin{tabular}{lccccr}
              \toprule
              Method & G & $\lambda$ & $p_c$ & $p_m$ & $p_z$\\
              \midrule
              Ours $(\delta = 73.43)$ & 150 & 1& 0.85 & 0.02 & 0.05\\
              Ours $(\delta = 146.87)$ & 150 & 1& 0.85 & 0.01 & 0.10\\
              Ours $(\delta = \infty)$ & 150 & 1& 0.85 & 0.001 & 0.15\\
             \bottomrule
            \end{tabular}
      \end{sc}
      \caption{Comparison parameters for MNIST dataset with digits 3\&2 (left) and 1\&4 (right).}
      \label{tab:mnist_comp_par}
    \end{scriptsize}
\end{table}
Table \ref{tab:comp_par} contains the hyperparameters used by our algorithm during the comparison.
\begin{table}[h]
    \centering
    \begin{scriptsize}
        \begin{sc}
        \setlength\tabcolsep{3.0pt}
          \renewcommand{\arraystretch}{0.8}
        \begin{tabular}{lccccr}
          \toprule
          Method & G & $\lambda$ & $p_c$ & $p_m$ & $p_z$\\
          \midrule
          Ours $(\delta = 0.15)$ & 20 & 1& 0.85 & 0.01 & 0.10\\
          Ours $(\delta = 0.30)$ & 20 & 1& 0.85 & 0.01 & 0.10\\
          \bottomrule
        \end{tabular}
      \end{sc}
    \end{scriptsize}
    \begin{scriptsize}
        \begin{sc}
        \setlength\tabcolsep{3.0pt}
          \renewcommand{\arraystretch}{0.8}
        \begin{tabular}{lccccr}
          \toprule
          Method & G & $\lambda$ & $p_c$ & $p_m$ & $p_z$\\
          \midrule
          Ours $(\delta = 10)$ & 50 & 1& 0.85 & 0.15 & 0.20\\
          Ours $(\delta = 20)$ & 50 & 1& 0.85 & 0.15 & 0.20\\
         \bottomrule
        \end{tabular}
      \end{sc}
    \end{scriptsize}
    \caption{Comparison parameters for Seeds (left) and MoCap Hand Postures (right).}
    \label{tab:comp_par}
% \caption{Comparison parameters for Seeds.}
% \label{tab:comp_par_seed}
\end{table}

\end{document}